\documentclass{article}
\usepackage{iclr2027_conference,times}
\usepackage{tabularx}
\usepackage{microtype}
\usepackage{graphicx}
\usepackage{subcaption}
\usepackage{booktabs}
\usepackage{tcolorbox}
\tcbuselibrary{skins}
\usepackage[ruled,vlined]{algorithm2e}

\usepackage{multirow}
\usepackage{adjustbox}
\usepackage{xcolor}
\usepackage{colortbl}
\usepackage{wrapfig}
\usepackage{amsmath,amssymb,mathtools,amsthm}
\usepackage{amsmath,amsfonts,bm}

\def\eqref#1{equation~\ref{#1}}
\def\1{\bm{1}}

\def\rvepsilon{{\mathbf{\epsilon}}}

\def\vg{{\bm{g}}}

\def\mW{{\bm{W}}}

\DeclareMathAlphabet{\mathsfit}{\encodingdefault}{\sfdefault}{m}{sl}
\SetMathAlphabet{\mathsfit}{bold}{\encodingdefault}{\sfdefault}{bx}{n}

\def\gN{{\mathcal{N}}}

\newcommand{\R}{\mathbb{R}}

\DeclareMathOperator*{\argmax}{arg\,max}
\DeclareMathOperator*{\argmin}{arg\,min}

\usepackage{hyperref}
\usepackage{url}
\usepackage[capitalize,noabbrev]{cleveref}
\setcitestyle{numbers,square,comma}

\hypersetup{
    colorlinks=false,
    citebordercolor={0 1 0},
    pdfborder={0 0 1}
}
\theoremstyle{plain}
\newtheorem{theorem}{Theorem}[section]

\newtheorem{corollary}[theorem]{Corollary}
\newtheorem{statement}[theorem]{Statement}
\theoremstyle{definition}

\theoremstyle{remark}

\definecolor{findingblue}{HTML}{0064E0}
\definecolor{takeawayred}{HTML}{C80A28}

\newcounter{fcounter}
\newcommand\finding[1]{%
  \refstepcounter{fcounter}\vspace{-2pt}%
  \begin{tcolorbox}[
    enhanced,
    colback=findingblue!10!white,
    frame hidden,            % no full border
    borderline west={2pt}{0pt}{blue!50!black}, % left-only blue bar
    boxsep=1pt,
    left=4pt,right=2pt,top=1pt,bottom=1pt,
    arc=0pt, outer arc=0pt,  % sharp corners
  ]
  \noindent{\textbf{\fontsize{10pt}{12pt}\selectfont Takeaway }: %
  \fontsize{10pt}{12pt}\selectfont #1}
  \end{tcolorbox}\vspace{-2pt}%
}
\crefname{fcounter}{Finding}{Findings}

\newcounter{kcounter}
\newcommand{\setequationspacing}{%
    \setlength{\abovedisplayskip}{6pt plus 1pt minus 1pt}%
    \setlength{\belowdisplayskip}{6pt plus 1pt minus 1pt}%
    \setlength{\abovedisplayshortskip}{6pt plus 1pt minus 1pt}%
    \setlength{\belowdisplayshortskip}{6pt plus 1pt minus 1pt}%
}

\AtBeginDocument{\setequationspacing}
\AddToHook{cmd/normalsize/after}{\setequationspacing}

\crefname{kcounter}{Takeaway}{Takeaways}

\title{Efficient Long-Horizon Learning for Learned Optimization}

\author{%
\textbf{Xiaolong Huang}$^{1,3}$ \quad \textbf{Benjamin Thérien}$^{1, 4}$ \quad \textbf{James Harrison}$^{2}$ \quad \textbf{Eugene Belilovsky}$^{1,3}$ \\
$^1$Mila - Quebec AI Institute \quad $^2$Google DeepMind \\$^3$Concordia University \quad $^4$Université de Montréal \\
}

\newif\ifarxiv
\arxivtrue   % arXiv 版；双盲投稿版改为 \arxivfalse
\ifarxiv\iclrfinalcopy\fi

\begin{document}
\maketitle
\ifarxiv\lhead{Preprint}\fi

\begin{abstract}
Learned optimization aims to improve upon hand-designed optimizers (e.g., Adam and Muon) by meta-learning small neural network optimizers over a distribution of tasks. While recent work has greatly advanced the architectural design and inductive biases of learned optimizers (LOs), their meta-training remains biased toward short-unroll learning on particular tasks, resulting in redundant computation and leaving LOs often unable to compete with hand-designed optimizers. We introduce Efficient Long-hOrizon (ELO) learning, an efficient meta-training algorithm that (1) reallocates wasted meta-training compute to longer failure regimes, achieving efficient long-horizon learning, and (2) enforces decoupled progressive expert supervision, providing stable meta-learning signals that additionally improve the generalization of LOs. 
Our empirical study evaluates ELO for meta-training both element-wise and matrix-based LOs. Across downstream language modeling (GPT-2 and GPT-3 on FineWeb) and image classification (ViT-B/16, ResNet-50 on ImageNet-1K) tasks, ELO substantially improves the long-unroll performance and out-of-distribution generalization of the base LOs. In particular, ELO-Celo2 outperforms well-tuned AdamW and Muon on GPT-2 (30M and 124M) language pretraining. \textit{Notably, all ELO baselines require less than 7 H100 GPU-hours for meta-training.}
\end{abstract}

\newcommand{\boldtheta}{{\boldsymbol{\theta}}}
\newcommand{\boldepsilon}{\boldsymbol{\epsilon}}
\newcommand{\boldxi}{\boldsymbol{\xi}}

\newcommand{\bolds}{\boldsymbol{s}}
\newcommand{\boldx}{\boldsymbol{x}}
\newcommand{\bolda}{\boldsymbol{a}}
\newcommand{\boldb}{\boldsymbol{b}}
\newcommand{\boldc}{\boldsymbol{c}}

\newcommand{\boldh}{\boldsymbol{h}}
\newcommand{\boldg}{\boldsymbol{g}}
\newcommand{\boldp}{\boldsymbol{p}}
\newcommand{\boldq}{\boldsymbol{q}}
\newcommand{\boldv}{\boldsymbol{v}}

\newcommand{\boldzero}{\boldsymbol{0}}
\newcommand{\boldone}{\boldsymbol{1}}

\newcommand{\hatg}{\hat{\boldg}}
\newcommand{\ges}{\hat{\boldg}^{\text{ES}}}
\newcommand{\gpes}{\hat{\boldg}^\text{PES}}

\newcommand{\dd}{{\mathrm{d}}}
\newcommand{\pd}[2]{\frac{\partial #1}{\partial #2}}
\newcommand{\pdn}[3]{\frac{\partial^#1 #2}{\partial #3^#1}}
\newcommand{\od}[2]{\frac{\dd #1}{\dd #2}}
\newcommand{\odn}[3]{\frac{\dd^#1 #2}{\dd #3^#1}}
\newcommand{\avg}[1]{\left< #1 \right>}
\newcommand{\pp}[1]{\left( #1 \right)}
\newcommand{\mb}{\mathbf}
\newcommand{\mx}{\mathbf x}
\newcommand{\mc}{\mathcal}
\newcommand{\bb}{\mathbb}
\renewcommand{\argmin}{\operatornamewithlimits{argmin}}
\renewcommand{\argmax}{\operatornamewithlimits{argmax}}
\newcommand{\median}{\operatornamewithlimits{median}}
\newcommand{\norm}[1]{\left|\left| #1 \right|\right|}
\newcommand{\expect}[2]{\mathbb{E}_{#1}\left[ #2 \right]}
\renewcommand{\R}[1]{\mathbb{R}^{#1}}
\newcommand{\ovec}[1]{\operatorname{vec}\pp{#1}}
\makeatletter
\def\gpesanti{%
  \@ifnextchar^%
    {\@gpesanti}
    {\@gpesanti^{}}%
}
\def\@gpesanti^#1{%
  % \hat{\boldg}^{{\text{PES-A}}^{#1}}%
  \hat{\boldg}^{\text{PES-A} #1}%
}
\makeatother
\newcommand{\veps}{\operatorname{vec}(\boldepsilon)}
\newcommand{\vTheta}{\operatorname{vec}(\Theta)}
\newcommand{\gesanti}{\hat{\boldg}^{\text{ES-A}}}

\definecolor{bestaccblue}{HTML}{0064E0}
\newcommand{\bestacc}[1]{\cellcolor{bestaccblue!20}\(\textbf{#1}\)}
\renewcommand{\arraystretch}{1.2}
\newtcolorbox{takeawaybox}{
    colback=bestaccblue!10,
    colframe=blue!40!black,
}

\section{Introduction}
The remarkable achievements of deep neural networks are closely tied to the evolution of optimization techniques~\citep{sun2019survey, sun2020optimization, abdulkadirov2023survey}. While hand-designed optimizers (e.g., Adam, Muon, etc.) remain the dominant choice in practice, learned optimizers (LOs), which use a small neural network to estimate parameter updates, have demonstrated the potential to discover superior optimization rules through meta-learning~\citep{andrychowicz2016learning,metz2019understanding,strongerbaselines,moudgil2025celo,moudgil2026celo2, therien2024mulo}.

\begin{figure}[t]
    \centering
    \includegraphics[width=1.0\linewidth]{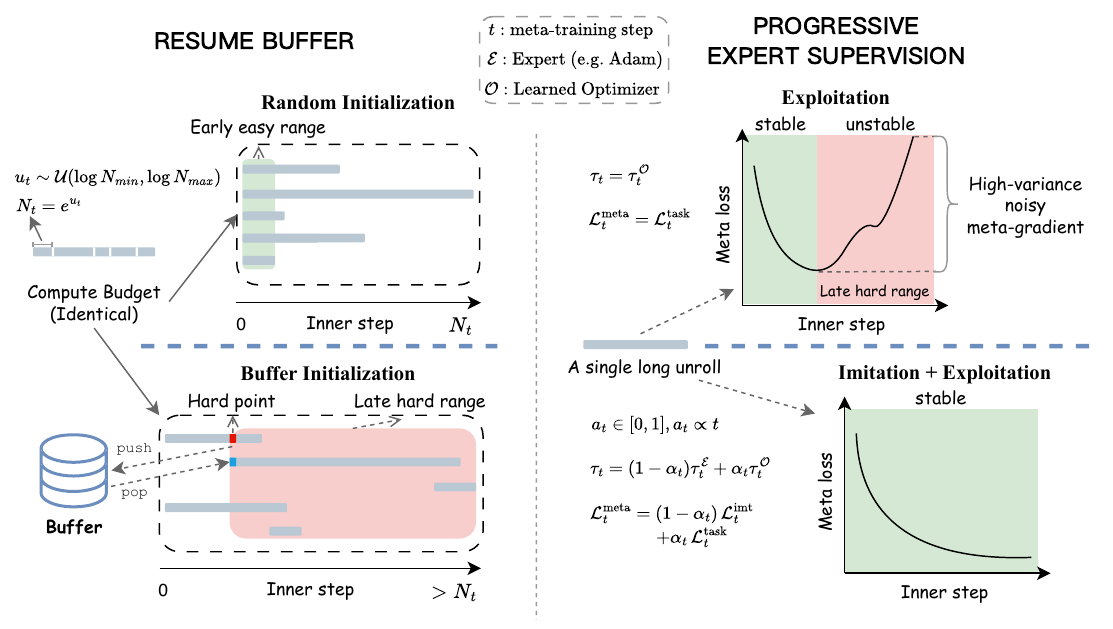}
    \caption{
Overview of ELO, which combines a resume buffer with progressive expert supervision.
\textbf{Left}: Under the same compute budget, the resume buffer restarts from stored states rather than random initialization, redirecting computation from easily learned early stages to later, failure-prone regions.
\textbf{Right}: Progressive expert supervision begins with pure offline imitation learning, then alternates between imitation and exploitation before switching entirely to exploitation, stabilizing meta-training throughout. While imitation learning potentially improves the LO's task-agnostic generalization, the final exploitation stage encourages it to learn update rules that outperform the expert.
}
    % \caption{
    % \textbf{ELO improves downstream optimization and meta-training efficiency.}
    % \textbf{(A)} Validation loss on GPT-2 350M pretraining with FineWeb. 
    % Meta-trained only on four tiny vision-MLP tasks, ELO-Celo2 is able to generalize to realistic language modeling problem, outperforming AdamW and matching Muon. 
    % \textbf{(B)} Meta-evaluation loss during meta-training. 
    % ELO converges faster than other baselines over the 2K-step unroll covered by all methods, and achieves lower loss at the longer 10K-step horizon.
    % }
    \label{fig:elo_lm_and_meta_efficiency}
\end{figure}

A long line of work has sought to improve learned optimization mainly through architectural innovations and carefully designed inductive biases. Architecturally, LOs have evolved from LSTM-based update rules~\citep{andrychowicz2016learning} and hierarchical RNNs~\citep{wichrowska2017learned} to LSTM-MLP hybrids~\citep{metz2020tasksstabilityarchitecturecompute}, per-parameter MLPs~\citep{metz2022practical}, LSTM hypernetworks~\citep{metz2211velo}, and, more recently, per-parameter MLPs equipped with Newton-Schulz iterations~\citep{moudgil2026celo2}. Complementary work has improved the inductive bias of LOs by adding nominal descent terms, weight decay, and output preconditioning~\citep{harrison2022closer}, incorporating Hessian and local-entropy flatness-aware regularizers~\citep{yang2023learninggeneralizeprovablylearning}, and placing the LOs in the $\mu$P parameterization to improve generalization~\citep{therien2024mulo}. Despite steady progress in these directions, existing meta-training algorithms still struggle to scale efficiently to long-horizon optimization learning. Moreover, the resulting LOs often underperform comparable hand-designed optimizers.

In this work, we discuss why traditional meta-training pipelines waste computation. To address this inefficiency, we use a resume buffer to effectively recycle compute from redundant early-unroll learning to later regions where long-horizon failures happen. However, this also introduces unstable meta-training dynamics due to compounding errors in long-unroll optimization. For stabilization and generalization purposes, we additionally constrain the LO with progressive expert guidance, which anchors early parameter updates to a reliable expert optimizer through separate direction and magnitude supervision, and gradually relaxes this constraint as the LO improves. Our contributions can be summarized as follows:

\begin{itemize}
    \setlength\itemsep{0em}
    \item We propose efficient long-horizon learning (ELO), a novel meta-training algorithm for learned optimization, which incorporates a resume buffer and progressive expert guidance. We empirically demonstrate that ELO achieves faster meta-training convergence than prior approaches under the same compute budget.
    \item We establish the performance of LOs meta-trained with ELO on practical downstream image classification and language modeling tasks, where ELO baselines consistently improve over their corresponding base LOs. Most notably, ELO-Celo2 outperforms well-tuned AdamW across all evaluated tasks and remains competitive with Muon on language modeling.
    \item Through extensive ablation studies, we establish the contribution of ELO's key components, including the improvements from the buffer, the effect of expert optimizer choice, and the benefit of a decoupled imitation objective over a standard MSE loss.
\end{itemize}
% \ificlrfinal
Our code is open-sourced at: \href{https://github.com/xiaol827/ELO}{\textcolor{blue}{\underline{https://github.com/xiaol827/ELO}}}.
% \else
% Replace this sentence with an anonymous repository link if one is available.
% The code repository link is omitted for anonymous review.
% \fi
% \vspace{-5pt}
\section{Background}\label{sec:background}
We review the necessary background and related work on learned optimization to establish the context for our method. For an extended review of related literature, we refer the reader to Appendix~\ref{apdx:sec:relatedwork}.

\paragraph{Learned Optimizer Architecture.} In this work, we adopt two LO architectures: the state-of-the-art Celo2~\citep{moudgil2026celo2} and the standard per-parameter MLP baseline \texttt{small\_fc\_lopt}~\citep{metz2022practical}, which we refer to as \texttt{small\_fc} for brevity. \texttt{small\_fc} is a three-layer MLP with ReLU activations. It takes as input a feature vector (e.g., gradient, momentum, rms, etc.) for each parameter $\theta$ of the optimizee (the network currently being optimized) and outputs an update direction, $d$, and magnitude, $m$. That is, $f_\phi(\cdot)=[d,m]$, where $f_\phi$ is the LO and $\phi$ denotes its parameters. Concretely, the update is applied as follows:
\begin{equation}
    \theta_{n+1} = \theta_{n} - \Delta\theta_n 
    ;\;\;\;
    \Delta\theta_n = \lambda_1 d \times e^{\lambda_2 m},
    \label{eq:lo_update}
\end{equation}
where $\lambda_1$ and $\lambda_2$ are hyperparameters usually fixed to $0.001$. Building on the element-wise update rule of \texttt{small\_fc}, Celo2 introduces tensor-level orthonormalization as a key improvement for non-1D parameters in the optimizee. Specifically, for any weight matrix $\mW$ in the optimizee, the update is applied as:
\begin{equation}
    \mW_{n+1} = \mW_{n} - \text{NewtonShulz5}(\Delta_n)/\|\Delta_n\|_\text{rms}.
    \label{eq:celo2_update}
\end{equation}
\paragraph{Meta-training Objective.} In general, learning the meta-parameters, $\phi$, involves solving an optimization problem of the form:
\begin{equation}\label{eq:obj}
\min_{\phi} \; \mathbb{E}_{\tau \sim \mathcal{D}} \Big[ \sum_{n=1}^N \mathcal{L}^{\text{meta}}_n(\theta_n; \phi, \tau) \Big],
\end{equation}
where $\mathcal{D}$ is a distribution of tasks, $\tau$ specifies a task (e.g., optimizee initialization, dataset, objective). The objective seeks to minimize the sum of per inner step losses over the training horizon $N$~\citep{therien2024mulo}.

\paragraph{Unbiased Truncated Gradient Estimation.} 
Estimating meta-gradients of long unrolled computation graphs with backpropagation is prone to noise due to exploding or vanishing gradients~\citep{metz2019understanding}. As such, the learned optimization community has turned towards zeroth-order estimators to estimate meta-gradients, including full ES~\citep{metz2019understanding} and the more efficient truncated estimator PES~\citep{vicol2021pes}, which we adopt in this work. PES is an unbiased estimator of the Gaussian-smoothed objective~\citep{vicol2021pes}. It estimates meta-gradients over segments of $R$ inner steps, where $R \ll N$ and $N$ denotes the full inner-problem horizon. With antithetic samples, the PES estimator at truncation index $k$ is given by:
\begin{equation}\label{eq:pesgrad}
\hat \vg^{\mathrm{PES}}_k
=
\frac{1}{2\sigma^2 P}
\sum_{i=1}^{P}
\sum_{n=1}^{R}
\xi^{(i)}_k
\Delta \mathcal{L}^{\text{meta}}_{n,i,k} ;\;\;\; \Delta \mathcal{L}^{\text{meta}}_{n,i,k} = \mathcal{L}^{\text{meta}}_n(\phi + \rvepsilon^{(i,k)})
-
\mathcal{L}^{\text{meta}}_n(\phi - \rvepsilon^{(i,k)}).
\end{equation}
Where $P$ is the number of monte-carlo estimates or particles, $\xi_k^{(i)}=\sum^{k}_{j=1} \rvepsilon^{(i,j)}$ is an accumulator of perturbations, and $\rvepsilon^{(i,j)}$ is the current perturbations sampled from $\gN(0,\sigma^2)$. Since our resume buffer maintains the PES state, the meta-gradients resulting from resumed inner problems are also unbiased (See Appendix~\ref{apdx:pes} for more details).

\paragraph{Sampling Inner-Loop Horizons.}
\label{bg:traditional_sampling}
% A crucial but seldom discussed aspect of meta-training LOs is how to choose the horizon of the inner problems. 
Existing work samples the inner unroll horizon mainly through log-uniform sampling~\citep{metz2022practical,therien2024mulo,moudgil2025celo,moudgil2026celo2} or curriculum sampling~\citep{strongerbaselines}. In log-uniform sampling, each inner problem is assigned a horizon $N$ sampled according to:
\begin{equation}
    p(N) =
    \frac{1}{N \log (N_{\max}/N_{\min})},
    \qquad
    N \in (N_{\min}, N_{\max}] .
\end{equation}
In curriculum sampling, training starts from an initial horizon $N_{\mathrm{init}}$ and increases it over time:
\begin{equation}
    N_j =
    \min\left(
    N_{\max},
    N_{\mathrm{init}}
    +
    \left\lfloor \frac{j}{N_{\mathrm{period}}} \right\rfloor
    \Delta_N
    \right),
\end{equation}
where $j$ denotes the number of inner problem resets, $N_{\mathrm{period}}$ controls how often the horizon is increased, and $\Delta_N$ is the horizon increment. Larger values of $N_{\max}$, or more aggressive curriculum schedules, expose the LO to very long horizons since early meta-training, but also risk instability, similar to ELO-small\_fc-B in \textbf{Figure~\ref{fig:ablation}} (Top). Therefore, in practice, relatively conservative settings are commonly used (\textbf{Figure~\ref{fig:meta_train}} (Middle)).

\section{ELO: Efficient Long-hOrizon learning}
% In this section, we first discuss why traditional meta-training pipelines lead to compute wastage. We then describe how ELO addresses the imbalance and how its components work together to allow efficient and stable long-horizon meta-training. A pseudocode description of ELO can be found in \textbf{Algorithm}~\ref{alg1}.
\SetKwInput{KwInput}{Input}
\SetKwInput{KwInit}{Initialize}
\SetKwInput{KwNotations}{Notations}
\definecolor{algblue}{HTML}{0064E0}
\definecolor{algred}{HTML}{C80A28}

\newtcbox{\redbox}[1][]{
  colback=algred!20,
  colframe=algred!20,  % same as fill = no visible border
  arc=6pt,                   % corner radius
  boxrule=1pt,
  left=0pt, right=0pt, top=-2pt, bottom=-2pt,
  nobeforeafter,
  tcbox raise base,
  #1
}

\newtcbox{\bluebox}[1][]{
  colback=algblue!20,
  colframe=algblue!20,  % same as fill = no visible border
  arc=6pt,                   % corner radius
  boxrule=1pt,
  left=0pt, right=0pt, top=-2pt, bottom=-2pt,
  nobeforeafter,
  tcbox raise base,
  #1
}

\begin{figure}[t]
\centering
\begin{adjustbox}{valign=T}
\begin{minipage}{0.65\textwidth}
\SetAlgoSkip{}
\begin{algorithm}[H]
\caption{Efficient Long-hOrizon Learning (ELO)}
\label{alg1}
\SetAlgoLined
\For{$t=0,1,\ldots,T-1$}{
    Sample $u\sim\operatorname{Uniform}(0,1)$\;
    \eIf{$(u<P_{\mathcal B})\wedge(\mathcal B\neq\emptyset)$}{
        \redbox{$(K,s_K)\sim\operatorname{Uniform}(\mathcal B)$}\;
    }{
        $K=0,\ s_0\leftarrow$ randomly initialized inner state\;
    }
    \redbox{$V_K\leftarrow0$; $M_K\leftarrow0$; $n^\star\leftarrow K$}\;
    $\alpha_t=\frac{t}{T-1}$\;
    \For{$n=K,K+1,\ldots,K+N-1$}{
        $\theta_{n+1}^{\mathcal E}\gets\theta_n-\Delta\theta_n^{\mathcal E}$; \quad
        $\theta_{n+1}^{\mathcal O}\gets\theta_n-\Delta\theta_n^{\mathcal O}$\;
        \bluebox{$\theta_{n+1}\gets(1-\alpha_t)\theta_{n+1}^{\mathcal E}
        +\alpha_t\theta_{n+1}^{\mathcal O}$}\;
        \bluebox{$\mathcal L_n^{\mathrm{dir}}\leftarrow1-
        \frac{\Delta\theta_n^{\mathcal E}\cdot\Delta\theta_n^{\mathcal O}}
        {\|\Delta\theta_n^{\mathcal E}\|_2\|\Delta\theta_n^{\mathcal O}\|_2}$}\;
        \bluebox{$\mathcal L_n^{\mathrm{mag}}\leftarrow
        \left|\|\Delta\theta_n^{\mathcal E}\|_2-\|\Delta\theta_n^{\mathcal O}\|_2\right|$}\;
        \bluebox{$\mathcal L_n^{\mathrm{imt}}\leftarrow
        \lambda\mathcal L_n^{\mathrm{dir}}+(1-\lambda)\mathcal L_n^{\mathrm{mag}}$}\;
        \bluebox{$\mathcal L_n^{\mathrm{meta}}\leftarrow
        (1-\alpha_t)\mathcal L_n^{\mathrm{imt}}+\alpha_t\mathcal L_n^{\mathrm{task}}$}\;
        \redbox{$\Delta_n^\ell\leftarrow\mathcal L_{n-1}^{\mathrm{task}}-
        \mathcal L_n^{\mathrm{task}}\quad\text{if}\quad n>K\quad\text{else}\quad0$}\;
        \redbox{$V_{n+1}\leftarrow V_n+\Delta_n^\ell$}\;
    }
    \redbox{$n^\star\leftarrow\arg\max V_n$}\;
    \redbox{$\mathcal B\leftarrow\operatorname{Push}(\mathcal B,(n^\star,s_{n^\star}))$}\;
    $g_t\leftarrow\operatorname{PES}\left(\sum_{n=K}^{K+N-1}\mathcal L_n^{\mathrm{meta}}\right)$\;
    $\phi_{t+1}\leftarrow\operatorname{Update}(\phi_t,g_t)$\;
}
\end{algorithm}
\end{minipage}
\end{adjustbox}%
\hspace{0.2cm}%
\vrule width 0.4pt
\hspace{0.2cm}%
\begin{adjustbox}{valign=T}
\begin{minipage}{0.30\textwidth}
\raggedright
\setlength{\parindent}{0pt}
\textbf{Hyperparams:}\par
\smallskip
Resume probability $P_{\mathcal B}$,
buffer capacity $C$,
direction loss weight $\lambda$.\par
\medskip
\vspace{0.4cm}
\textbf{Initialize:}\par
\smallskip
Resume buffer $\mathcal B\leftarrow\emptyset$
with capacity $C$.\par
\medskip
\vspace{0.8cm}
\textbf{Notation:}\par
\smallskip
\textbf{ELO} combines \redbox{Resume Buffer} and
\bluebox{\shortstack[l]{Progressive Expert\\Supervision}}.\par
\medskip
\begingroup
\setlength{\tabcolsep}{3pt}
\renewcommand{\arraystretch}{1.1}
\begin{tabularx}{\linewidth}{@{}l>{\raggedright\arraybackslash}X@{}}
$T$ & meta-train steps \\
$N$ & unroll length \\
$K$ & inner start step \\
$\mathcal B$ & resume buffer \\
$s_n$ & inner state at step $n$ \\
$\mathcal E$ & expert optimizer \\
$\mathcal O$ & learned optimizer (LO) \\
$\theta$ & optimizee parameters \\
$\phi$ & LO parameters \\
$\alpha_t$ & progressive schedule \\
$V_n$ & progress score
\end{tabularx}
\endgroup
\end{minipage}
\end{adjustbox}
\end{figure}

\subsection{Inefficiency of Random Restarts}
During meta-training, when the LO makes an update error $\epsilon_n$ at inner step $n$, then the state reached at a later step $m > n$ depends on the accumulated effect of all previous errors:
\begin{equation}
    \theta_m - \theta_m^\star
    \approx
    \sum_{n=0}^{m-1} C_{n,m}\epsilon_n ,
\end{equation}
where $\theta_m^\star$ represents the ideal optimal state at step $m$, and $C_{n,m}$ describes how an error made at step $n$ propagates to the state at step $m$. 

With less accumulated error, the inner task loss (e.g., CE Loss for image classification) decreases more steadily early in an unroll, while later steps are more likely to diverge. This is especially common early in meta-training, when the LO has not yet learned a reliable update policy. Such divergence can produce large meta-gradients that distort the meta-optimizer's running statistics, such as momentum. Even if the loss recovers, these effects can persist and interfere with subsequent meta-updates. We therefore treat these gradients as high-variance noise. In contrast, the cleaner, more stable meta-gradients from early steps generally allow faster learning of early-stage behavior. 

Existing horizon sampling strategies, such as those described in \textbf{Sec.~\ref{bg:traditional_sampling}}, always restart each inner problem from random initialization, which biases meta-training toward early unroll steps. As a result, once early-unroll behavior has been well learned, random restarts continue to allocate a disproportionate fraction of the meta-training budget to these already well-covered regions, leaving later, failure-prone regimes under-trained.

\subsection{Resume Buffer}

Throughout meta-training, we maintain a resume buffer that lets future unrolls restart from challenging points in previous trajectories. Specifically, we identify these regions from the loss dynamics of the inner trajectory. Given inner losses $\ell_0, \ell_1, \ldots, \ell_N$, define the one-step loss improvement $\Delta^{\ell}_n = \ell_{n-1} - \ell_n$. We accumulate these local improvements into a running progress score, 
\begin{wrapfigure}{r}{0.3\textwidth}
    \centering
    \vspace{-10pt}
    \includegraphics[width=0.3\textwidth]{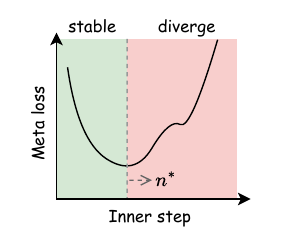}
    \caption{Push inner state to buffer at step $n^*$.
% Expert supervision reduces meta-gradient variance and improves the signal-to-noise ratio of the $\Delta$meta-loss between antithetic PES perturbations (See \textbf{Sec.~\ref{sec:background}}), leading to cleaner meta-gradient estimates.
% \begin{equation}
%     n^\star = \arg\max_{0 \le n \le N} V_n.
% \end{equation} 
}
    \vspace{-10pt}
    \label{fig:diverge_point}
\end{wrapfigure}
$V_n = V_{n-1} + \Delta^{\ell}_n$, with $V_0 = 0$, so that $V_n = \ell_0 - \ell_n$ measures the total progress made up to step $n$. After a full unroll, we select the step with the highest progress score, $n^\star = \arg\max_{0 \le n \le N} V_n.$ This step corresponds to the lowest loss reached during the unroll, after which the loss diverges or saturates. We then push the corresponding inner state $s_{n^\star}$ into the buffer. At the beginning of each new unroll, we randomly resume from one of the buffered states with probability $P_{\mathcal B}$ and otherwise start from a fresh random initialization. This shifts a controlled fraction of meta-training compute from already well-covered early stages to later regimes where long-horizon failures emerge. We set the buffer size to one by default, both for simplicity and to avoid stale states.

\subsection{Progressive Expert Supervision}
When applying aggressive resume buffer schedule (large $P_{\mathcal B}$), we observe unstable meta-training dynamics (ELO-small\_fc-B in \textbf{Figure~\ref{fig:ablation}} Top). As discussed in \textbf{Sec.~\ref{bg:traditional_sampling}}, a large resume probability $P_{\mathcal B}$ can expose the LO to these failure-prone regions early in meta-training. The resulting noisy meta-gradients (\textbf{Figure~\ref{fig:ablation}} (Top)) can destabilize meta-training and, in severe cases, lead to collapse (\textbf{Figure~\ref{fig:ablation} (Top)}).

To stabilize meta-training under a large $P_{\mathcal B}$ and encourage task-agnostic generalization learning, we regularize the LO with a reliable expert optimizer during early meta-training and progressively relax this constraint as the LO improves.

\paragraph{Trajectory fusion.} For any inner step $n$, we fuse the trajectory produced by the expert and the one produced by the LO in the following way:
\begin{equation}
\label{eq_trajectory}
\begin{gathered}
\theta_{n+1}^{\mathcal E}
= \theta_n - \Delta\theta_n^{\mathcal E};
\qquad
\theta_{n+1}^{\mathcal O}
= \theta_n - \Delta\theta_n^{\mathcal O};
\\
\theta_{n+1}
= (1-\alpha_t)\,\theta_{n+1}^{\mathcal E}
+ \alpha_t\,\theta_{n+1}^{\mathcal O},
\end{gathered}
\end{equation}
where $\mathcal{E}$ and $\mathcal{O}$ denote the expert and the LO, respectively. As meta-training progresses, the LO's updates become increasingly reliable; the schedule $\alpha_t$ correspondingly shifts weight from the expert to the LO, keeping the inner trajectory high-quality throughout meta-training.

\paragraph{Progressive Expert Supervision.} When applying mixed trajectories, meta-gradients at the beginning of meta-training become uninformative, as they are estimated from pure expert trajectories ($\alpha_t=0$), rather than trajectories induced by the LO itself. To address this, we leverage expert supervision to provide clean gradient signals for early meta-training. A natural choice is to imitate expert updates with a standard MSE loss, which couples errors in update direction and magnitude. We instead decouple direction and magnitude supervision and combine them with explicit weights for finer control over their relative importance, yielding the following learning objective:
\begin{equation}
\begin{gathered}
\mathcal{L}_n^{\mathrm{dir}}
= 1 -
\frac{
    \Delta\theta_n^{\mathcal E}
    \cdot \Delta\theta_n^{\mathcal O}
}{
    \left\|\Delta\theta_n^{\mathcal E}\right\|_2
    \left\|\Delta\theta_n^{\mathcal O}\right\|_2
};
\qquad
\mathcal{L}_n^{\mathrm{mag}}
=
\left|
    \left\|\Delta\theta_n^{\mathcal E}\right\|_2
    -
    \left\|\Delta\theta_n^{\mathcal O}\right\|_2
\right|;
\\
\mathcal{L}_n^{\mathrm{imt}}
=
\lambda \mathcal{L}_n^{\mathrm{dir}}
+
(1-\lambda)\mathcal{L}_n^{\mathrm{mag}}.
\end{gathered}
\end{equation}
We define the meta-loss as $\mathcal{L}^{\mathrm{meta}}_n = (1 - \alpha_t) \, \mathcal{L}^{\mathrm{imt}}_n + \alpha_t \, \mathcal{L}^{\mathrm{task}}_n$, where $\alpha_t$ follows the same schedule as in \textbf{Eq.}~\ref{eq_trajectory}. 
When $\alpha_t=0$, the objective reduces to pure expert supervision. 
As $\alpha_t$ increases, the objective continuously interpolates between imitation and task optimization. 
At $\alpha_t=1$, the LO is trained only through task-driven loss, allowing it to improve beyond the expert policy.

%%%%%%%%%%%%%%%%%%%%%%%%%%%%%%%%%%%%%%%%%%%%%%%%%%%%%%%%%%%%%%%
% Longer version 
%%%%%%%%%%%%%%%%%%%%%%%%%%%%%%%%%%%%%%%%%%%%%%%%%%%%%%%%%%%%%%%
\section{Empirical Evaluation}\label{ss:experiments}
We describe the main experimental settings for evaluating ELO, with complete details provided in Appendix~\ref{apdx:ss:experiments}.

\paragraph{Baselines.}
We adopt two representative MLP-based LOs, 
\texttt{small\_fc}~\citep{metz2022practical} and Celo2~\citep{moudgil2026celo2}. For each architecture, we compare the ELO-trained variant, ELO-\texttt{small\_fc} and ELO-Celo2, with its non-ELO counterpart. 
We also compare against CL-\texttt{small\_fc}, a curriculum baseline with alternating offline imitation learning inspired by~\citep{strongerbaselines}, and strong hand-designed optimizers including AdamW~\citep{loshchilov2017decoupled} and Muon~\citep{jordan2024muon}. 

\begin{figure}[t]
    \centering
    {\includegraphics[width=1.0\linewidth]{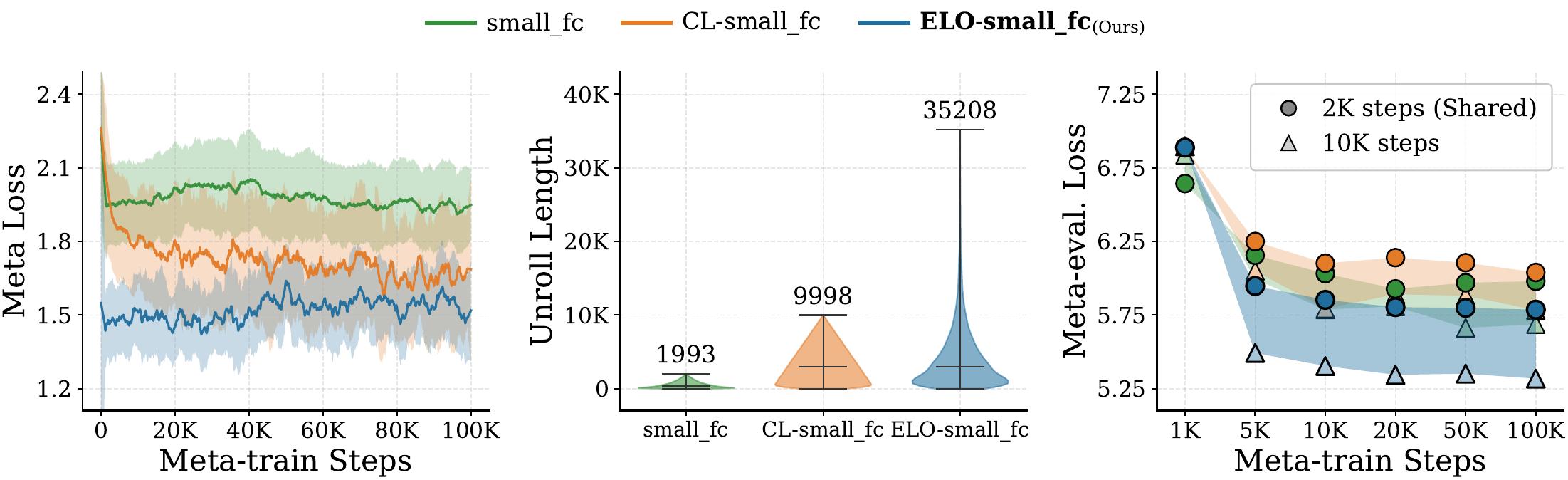}}
\caption{
The meta-loss at each outer step is the average task loss over a full unroll or truncation window. ELO maintains a lower meta-loss by enabling learning over longer unrolls under a fixed compute budget, and achieves faster meta-training convergence at both the shared 2K-step horizon and the longer 10K-step horizon.
}
\label{fig:meta_train}
\end{figure}

\textbf{Meta-training.} Our experimental setup and meta-training pipeline largely follow~\citep{metz2022practical, therien2024mulo,moudgil2025celo} with tuned hyperparameters applied for each baseline. All LOs are meta-trained on four simple $8 \times 8$ image classification tasks: MNIST, Fashion-MNIST, CIFAR-10, and SVHN. The optimizee is a one-layer MLP with width 32. We meta-train for 100K outer steps with batch size 64, using AdamW as the meta-optimizer with a cosine learning rate schedule. We search the outer learning rate over $\{10^{-3}, 3 \times 10^{-4}, 10^{-4}\}$ and use a weight decay of $10^{-4}$. We estimate meta-gradients with Persistent Evolution Strategies (PES), using a truncation length of $R=50$. We also apply task augmentation with range of [0.001, 1000] during meta-training, following~\citep{moudgil2025celo}.

\paragraph{Downstream Evaluation.}
We evaluate ELO on various practical vision and language model pre-training tasks. For vision experiments, We train on ResNet-50~\citep{he2016deep} and ViT-Base/16~\citep{dosovitskiy2020image} on ImageNet-1K. For language modeling, we train GPT-2 (30M and 124M) and GPT-3 (1.3B) on FineWeb~\citep{penedo2024fineweb}, using a budget of 20 tokens per parameter for the GPT-2 models~\citep{chinchilla} and 10B tokens for GPT-3. All methods use cosine learning rate schedules, and we carefully tune the learning rate and weight decay for each optimizer family.

\subsection{Generalization to Practical Pre-training Tasks}
\finding{ELO-LOs meta-trained at a small scale are able to generalize to models with hundreds of millions of parameters, consistently improving over their corresponding base LOs across multiple benchmarks. Most notably, ELO-Celo2 outperforms well-tuned AdamW across all tasks and remains competitive with Muon on language modeling at different scales.}

\begin{figure}[t]
    \centering
    \includegraphics[width=1.0\linewidth]{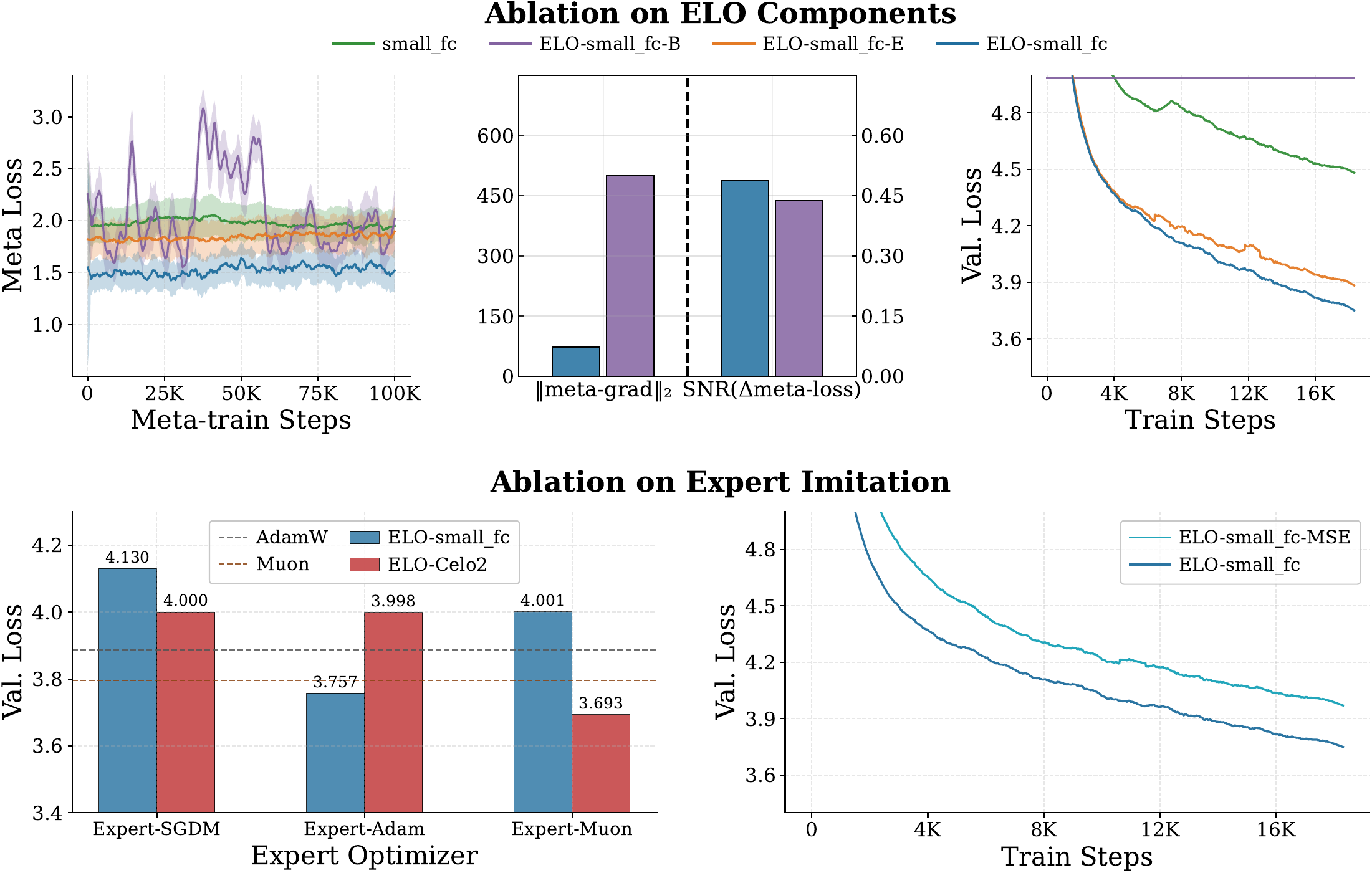}
    \caption{
Ablations on meta-training task (four tiny image classification tasks using a one-layer MLP) and downstream GPT-2 (30M) pretraining for 0.6B tokens.
\textbf{Top}: An aggressive resume buffer schedule leads to  larger meta-gradients with lower signal-to-noise ratio. Expert supervision stabilizes training, while the resume buffer further improves downstream long-unroll performance.
\textbf{Bottom}: For expert optimizer, AdamW works best for the element-wise \texttt{small\_fc}, while Muon works best for the matrix-based Celo2. Separately weighting direction and magnitude supervision also improves downstream performance over standard MSE imitation.
}
    \label{fig:ablation}
\end{figure}

All ELO baselines are meta-trained on four simple vision tasks using a one-layer MLP with width 32. We evaluate whether optimizers learned in this small meta-training setting can transfer to realistic optimization problems. We compare ELO-trained LOs with other LO baselines and hand-designed optimizers on various significantly larger scale vision and language model pre-training tasks.

\textbf{Image classification.} We train ViT-B/16 and ResNet-50 models on ImageNet-1K at $224 \times 224$ resolution for 50K steps with batch size 2048. We report top-1 validation accuracy in \textbf{Table}~\ref{tab:img_results} on three validation benchmarks: ImageNet-1K validation, ImageNet-ReaL, and ImageNet-V2. LOs meta-trained with ELO achieve stronger performance than all LO baselines using the same architecture and outperform a strong AdamW baseline. As shown in \textbf{Figure}~\ref{fig:aug_bars}, we find that LOs require substantially stronger data augmentation than AdamW. Under the augmentation setting well tuned for AdamW, LOs tend to overfit, while increasing augmentation strength improves their validation performance. In contrast, AdamW naturally generalizes better under a lighter augmentation, and its performance degrades under stronger augmentation.

% \textbf{Language model pre-training.}
% We evaluate ELO on realistic out-of-distribution language modeling tasks (\textbf{Figure}~\ref{fig:fineweb_results} and \textbf{Figure}~\ref{fig:elo_lm_and_meta_efficiency} A). We first consider GPT-2-small (124M) trained for 2.5B tokens, where both ELO-small\_fc and ELO-Celo2 substantially improve over their corresponding base LOs. In this setting, ELO-Celo2 achieves the best performance, with a small advantage over Muon. We further evaluate on a larger-scale setting with GPT-2-medium (350M) trained for 7B tokens. ELO-Celo2 continues to significantly outperform AdamW and performs on par with Muon. These evidences show that ELO-trained LOs transfer effectively to language modeling tasks that are substantially different from those used during meta-training.

\textbf{Language model pretraining.}
We evaluate ELO on out-of-distribution language modeling tasks on FineWeb (\textbf{Figure}~\ref{fig:fineweb_results}). On GPT-2 (30M and 124M), trained for 0.6B and 2.5B tokens, respectively, both ELO-\texttt{small\_fc} and ELO-Celo2 substantially improve over their corresponding base LOs. ELO-Celo2 achieves the lowest validation loss at both scales, even outperforming Muon. We further evaluate ELO-Celo2 on GPT-3 (1.3B), trained for approximately 10B tokens, comparing it with strong learned and hand-designed baselines, including Celo2, AdamW, and Muon. While Muon achieves the lowest validation loss, ELO-Celo2 shows substantially better generalization than the base Celo2.

\subsection{ELO Design Analysis}

\finding{Resume buffer improves long-horizon coverage, while expert supervision keeps these harder unrolls stable. The optimal expert is architecture dependent, with Adam preferred for \texttt{small\_fc} and Muon preferred for Celo2. Decoupling direction and magnitude supervision further outperforms direct $\ell_2$ regression.}

In this section, we conduct ablation and comparison experiments to understand how each design of ELO contributes and how they interact.

% \textbf{Ablating ELO's main components.}
% ELO consists of two main components: the resume buffer and progressive expert supervision. When applying only the resume buffer with a aggressive schedule, meta-training becomes unstable (\textbf{Figure}~\ref{fig:ablation} (Top), ELO-small\_fc-B). This instability arises because the buffer exposes the LO to long-horizon inner problems before it has learned reliable updates, causing errors to accumulate along the inner trajectory. In contrast, using only progressive expert supervision generalizes well to language model pretraining, but does not sufficiently improve long-horizon optimization (\textbf{Figure}~\ref{fig:ablation} (Top), ELO-small\_fc-E). Combining the two components gives stable meta-training and strong long-horizon optimization performance (\textbf{Figure}~\ref{fig:ablation} (Top), ELO-small\_fc).

\textbf{Ablating ELO's main components.}
We ablate the resume buffer and progressive expert supervision in \textbf{Figure}~\ref{fig:ablation} (Top). Using only the resume buffer with an aggressive schedule (ELO-\texttt{small\_fc}-B) destabilizes meta-training and leads to poor generalization to downstream task. The buffer exposes the LO to long-horizon inner problems before it has learned reliable updates, allowing errors to accumulate along the trajectory. Compared with full ELO, this variant produces much larger meta-gradient norms and a lower signal-to-noise ratio in the meta-loss differences (closely related to the quality of meta-gradient estimates) used by PES. In contrast, progressive expert supervision alone (ELO-\texttt{small\_fc}-E) stabilizes meta-training and substantially improves downstream performance. Combining both components (ELO-\texttt{small\_fc}) preserves this stability and further improves downstream optimization, with a growing advantage over expert supervision alone at later training steps.

\textbf{Expert optimizer selection.}
The choice of expert optimizer is important because the expert defines the update target used during imitation and is the only source of learning signal early on during meta-training. We search over several hand-designed optimizers independently for different LO architectures and compare the downstream performance of resulting optimizers on language model pre-training with GPT-2 (30M), trained for 0.6B tokens. As shown in \textbf{Figure}~\ref{fig:ablation} (Bottom), Adam gives the best validation loss when used as the expert for the element-wise \texttt{small\_fc}. In contrast, Muon works best as the expert for the matrix-based Celo2 optimizer. A likely explanation is that \texttt{small\_fc} applies element-wise update rules and does not explicitly model matrix operations, making it difficult to imitate Muon's orthogonalization based update. Celo2, by contrast, contains matrix-based operations and can more naturally imitate Muon-like behavior.

\begin{figure}[t]
    \centering
    {\includegraphics[width=1.0\linewidth]{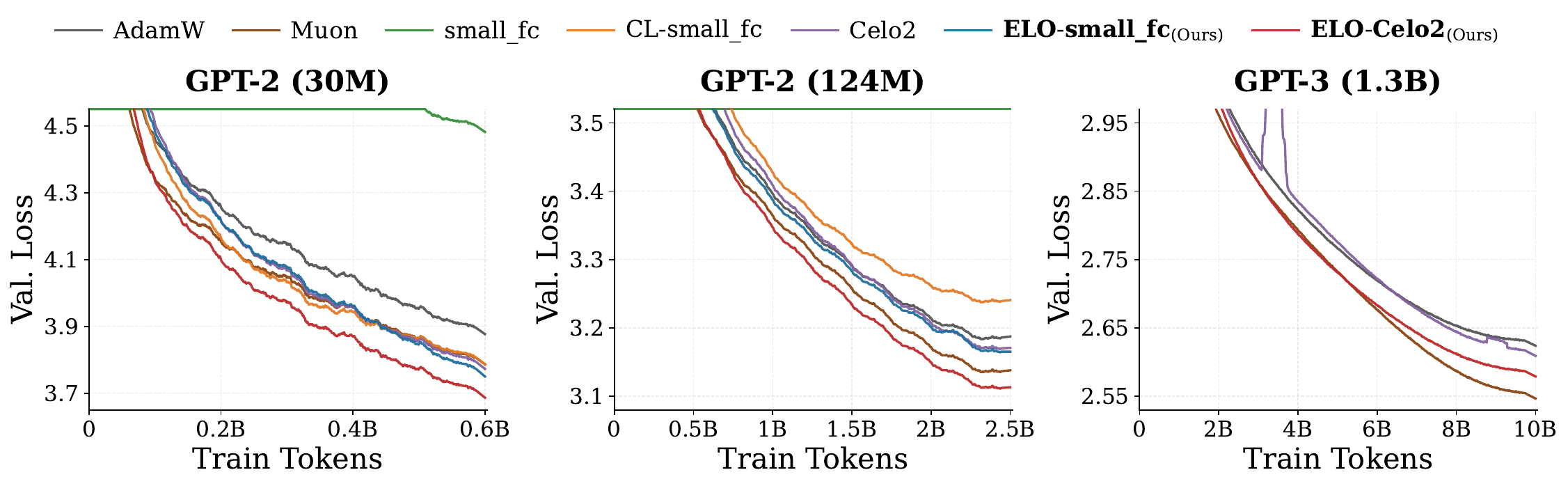}}
\caption{
Validation loss of downstream GPT model pretraining on FineWeb. We compare all methods on GPT-2 (30M and 124M), while the GPT-3 (1.3B) comparison focuses on four strong optimizers: AdamW, Muon, Celo2, and ELO-Celo2. In all cases, ELO-trained optimizers substantially outperform their corresponding base LOs (\texttt{small\_fc} and Celo2).
}
\label{fig:fineweb_results}
\end{figure}

\textbf{Imitation objective.}
We compare our decoupled imitation objective with direct MSE regression between the updates produced by the LO and the expert optimizer. While MSE couples update direction and magnitude, our approach allows finer control over their relative weights. This leads to lower validation loss in downstream task (GPT-2 (30M) pretraining in \textbf{Figure}~\ref{fig:ablation} (Bottom)).

\subsection{ELO Hyperparameter Sensitivity}

% \begin{takeawaybox}
% \textbf{Key Takeaway:}
% Imitating expert's update direction is more important than matching its magnitude. A large buffer probability is preferred in our meta-training settings.
% \end{takeawaybox}
\finding{Imitating expert's update direction is more important than matching its magnitude. A large buffer probability is preferred in our meta-training settings.}

ELO introduces two critical hyperparameters: 1) the buffer resume probability $P_{\mathcal B}$, which controls how often unrolls are initialized from buffered states, and 2) the direction loss weight $\lambda$, which balances direction and magnitude supervision. We tune both hyperparameters mainly using ELO-\texttt{small\_fc}.

\textbf{Direction vs. magnitude.} We set $P_{\mathcal B}=0$ to isolate the imitation objective from the buffer mechanism. We sweep the direction loss weight $\lambda$ over $\{0.1,0.2,\ldots,0.9\}$ and select it according to meta-evaluation loss. As shown in Appendix~\ref{fig:dir_mag_sweep}, assigning weight $0.7$ to the direction term and $0.3$ to the magnitude term gives the best performance. This suggests that matching the expert update direction is more important than exactly matching its scale, although magnitude supervision remains useful for controlling the update size.

\textbf{Buffer probability.} $P_{\mathcal{B}}$ controls how aggressively meta-training compute is reallocated toward longer and harder regimes, which is critical for downstream long-horizon tasks. Larger values naturally increase the probability of training from later inner states, but can also make meta-training less stable. We search $P_{\mathcal B}$ over $\{0.1,\ldots,0.9\}$. To reduce the sweep cost while selecting a value that generalizes to downstream tasks, we conduct a two-stage selection. We first select three candidates, $P_{\mathcal B}\in\{0.1,0.5,0.8\}$, based on meta-evaluation loss. We then select the final $P_{\mathcal B}=0.8$ according to the validation loss of GPT-2 (124M, 2.5B tokens) pretraining in Appendix~\ref{fig:buffer_prob_sweep}. In practice, smaller $P_{\mathcal B}$ is safer at small meta-training scale, while larger $P_{\mathcal B}$ becomes beneficial as meta-training becomes more stable and can support longer continuations.

\begin{figure}[t]
    \centering
    \includegraphics[width=\linewidth]{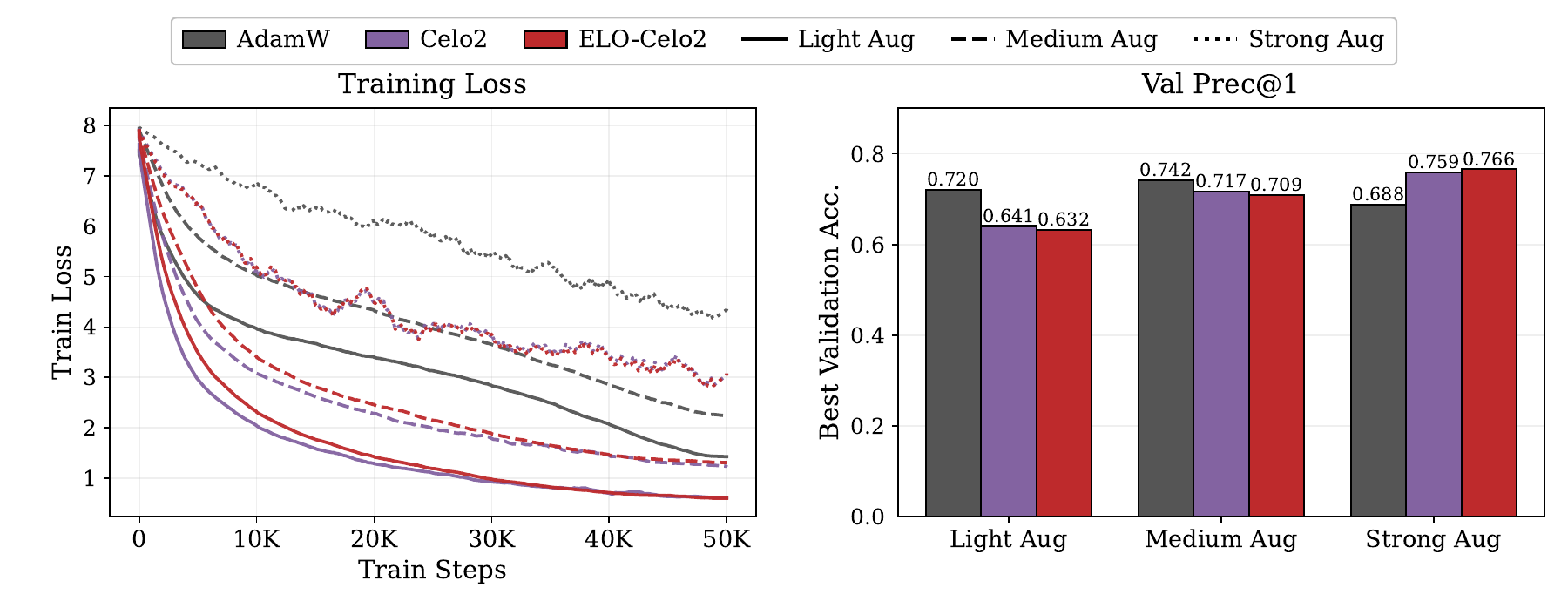}
    \caption{
    On vision tasks (e.g. ViT-B/16 on ImageNet-1K), learned optimizers achieve their best validation accuracy under stronger augmentation, using 6 randomly selected operations from a pool of more than 20, whereas hand-designed optimizers perform best under medium augmentation, using 4 randomly selected operations.
    }
    \label{fig:aug_bars}
\end{figure}

\begin{table}[t]
  \centering
  \caption{
    Best validation accuracy across benchmarks.
    Top-1 accuracy on ImageNet-1K validation, ImageNet-ReaL, and ImageNet-V2 for ResNet-50 and ViT-B/16 trained from scratch, mean $\pm$ std over 3 seeds. The best mean in each column is in \textbf{bold}.
    }
  \label{tab:img_results}
  \begin{adjustbox}{width=1.0\textwidth}
  \begin{tabular}{lccc|ccc|c}
  \toprule
   & \multicolumn{3}{c}{ResNet-50} & \multicolumn{3}{c}{ViT-B/16} & \\
  Method & IN-1K & IN-Real & IN-V2 & IN-1K & IN-Real & IN-V2 & Avg. \\
  \midrule
  AdamW                                        & 75.81{\scriptsize$\pm$0.07} & 82.17{\scriptsize$\pm$0.02} & 63.39{\scriptsize$\pm$0.15} & 74.03{\scriptsize$\pm$0.11} & 80.72{\scriptsize$\pm$0.18} & 61.30{\scriptsize$\pm$0.27} & 72.90 \\
  small\_fc                                    & 65.31{\scriptsize$\pm$0.10} & 73.66{\scriptsize$\pm$0.03} & 53.51{\scriptsize$\pm$0.11} & 32.10{\scriptsize$\pm$0.94} & 37.09{\scriptsize$\pm$1.10} & 25.69{\scriptsize$\pm$0.88} & 47.89 \\
  CL-small\_fc                                 & 71.91{\scriptsize$\pm$0.30} & 79.23{\scriptsize$\pm$0.26} & 59.58{\scriptsize$\pm$0.30} & 67.65{\scriptsize$\pm$0.22} & 74.88{\scriptsize$\pm$0.16} & 54.73{\scriptsize$\pm$0.27} & 68.00 \\
  Celo2                                        & 75.76{\scriptsize$\pm$0.11} & 82.34{\scriptsize$\pm$0.07} & 63.39{\scriptsize$\pm$0.27} & 75.86{\scriptsize$\pm$0.05} & 81.74{\scriptsize$\pm$0.05} & 62.76{\scriptsize$\pm$0.09} & 73.64 \\
  \midrule
  $\textbf{ELO-small\_fc}_{(\text{Ours})}$     & 74.08{\scriptsize$\pm$0.16} & 81.13{\scriptsize$\pm$0.12} & 62.10{\scriptsize$\pm$0.30} & 75.15{\scriptsize$\pm$0.07} & 81.69{\scriptsize$\pm$0.05} & 62.08{\scriptsize$\pm$0.10} & 72.70 \\
  $\textbf{ELO-Celo2}_{(\text{Ours})}$         & \textbf{76.11}{\scriptsize$\pm$0.09} & \textbf{82.70}{\scriptsize$\pm$0.02} & \textbf{64.03}{\scriptsize$\pm$0.39} & \textbf{76.42}{\scriptsize$\pm$0.28} & \textbf{82.26}{\scriptsize$\pm$0.31} & \textbf{63.89}{\scriptsize$\pm$0.32} & \textbf{74.24} \\
  \bottomrule
  \end{tabular}
  \end{adjustbox}
\end{table}

% \section{Discussion \& Limitation}
% Through extensive experiments, we show that ELO further improves the capability of LOs and enables them to outperform state-of-the-art hand-designed optimizers on multiple realistic downstream tasks. However, the preferred resume probability may depend on the scale of meta-training. In future work, we plan to moderately expand the meta-training scale and study whether a clear scaling law emerges for this hyperparameter. A remaining limitation of current LOs is their longer per-step runtime compared with hand-designed optimizers, mainly due to their per-parameter update design. Although this is not the focus of this paper, we are also interested in exploring more efficient LO architectures in future work and meta-training them with ELO, with the goal of building LOs that can replace hand-designed optimizers in a broader range of practical settings.

\section{Discussion \& Limitation}
% While we have made our best effort to provide a thorough scientific study and taken the necessary measures to ensure our results are sound, limitations naturally arise from the computational resources available in academia. In no particular order, our empirical results do not evaluate models beyond 350M scale, LM architectures beyond standard dense transformers (e.g.\ MoEs, efficient-attention variants, or hybrid architectures), or downstream tasks beyond image classification and language modeling. In addition, while we compare against strong hand-designed optimizers such as AdamW and Muon, we do not evaluate additional structured or second-order-inspired optimizers such as SOAP or Shampoo. On the implementation side, ELO can introduce additional overhead compared with standard meta-training recipes (Appendix~\ref{tab:meta_train_runtime}) due to constraints of the JAX framework, where more data movement is required across nested functions. This overhead is more visible on tiny meta-training tasks, where GPU computation is small and CPU and memory operations become the main bottleneck. As the task scale increases, however, runtime becomes dominated by GPU computation, and ELO’s relative overhead is expected to decrease substantially.

Due to computational constraints, we have not explored large-scale meta-training, which we believe is crucial for broader LO generalization. Our LOs are meta-trained on only four tiny image classification tasks, and we observe that they generalize less well across model scales than strong hand-designed optimizers. However, we believe that combining ELO with parameterizations such as $\mu$P~\citep{therien2024mulo} or CompleteP~\citep{dey2026don} has big potential to improve generalization across model scales under limited meta-training budgets. On the implementation side, ELO can introduce additional overhead compared with standard meta-training recipes (Appendix~\ref{tab:meta_train_runtime}), mainly due to additional data movement across nested functions in JAX implementation. This overhead is more visible on tiny meta-training tasks, where GPU computation is small and CPU and memory operations become the main bottleneck. As the task scale increases, however, runtime becomes dominated by GPU computation, and ELO's relative overhead is expected to decrease substantially.

\section{Conclusion}
We proposed Efficient Long-hOrizon learning (ELO), a novel meta-training algorithm for learned optimization that combines progressive expert guidance with a resume buffer to achieve efficient long-horizon learning with stability and improve downstream performance. Across various vision and language-model benchmarks, ELO-small\_fc and ELO-Celo2 outperform their corresponding learned-optimizer baselines, with ELO-Celo2 also outperforming AdamW and Muon on GPT2 (30M and 124M) pre-training. These results suggest that ELO is a promising recipe for scaling learned optimization.

\clearpage
% This statement records only AI assistance established during this migration.
% Authors: complete it with any other AI use in the research or manuscript,
% and confirm your review and responsibility before submitting the paper.
\section*{AI Use Statement}
We used generative AI to help polish the writing, debug code, and prepare figures. These uses fall under the recommended disclosure categories in the ICLR 2027 AI Policy for Authors. We did not use generative AI for any tasks requiring disclosure. The research hypotheses, ELO design, experimental design, method implementations, and interpretation of results are our own work. This work did not involve synthetic data generation, dataset cleaning, translation, or qualitative data analysis. We take responsibility for the entire paper and its accompanying artifacts, including any AI-assisted content.

\section*{Ethics Statement}
We study meta-training algorithms for learned optimizers. We do not involve human subjects in our experiments, nor do we collect or release new datasets. We use publicly available datasets, including MNIST, Fashion-MNIST, CIFAR-10, SVHN, ImageNet-1K, ImageNet-ReaL, ImageNet-V2, and FineWeb, in accordance with their respective licenses and usage conditions.
We evaluate optimization performance on standard image classification and language modeling tasks. We recognize that publicly available datasets may still raise concerns about content, provenance, and bias, and we do not claim that our experiments resolve these concerns. We also recognize that the broader impact of our method depends on the models and applications in which it is used. We therefore encourage researchers and practitioners adopting our method to consider their data sources, intended uses, and potential effects on individuals and communities.

\section*{Reproducibility Statement}
We describe the ELO meta-training procedure in Algorithm~1 and its main components in Secs.~3.2--3.3. Appendix~A provides the assumptions and analysis for the unbiasedness of PES meta-gradients when resuming inner problems from the buffer. Experimental details, including task and model configurations, training budgets, and hyperparameter searches, are given in Sec.~4 and Appendix~D. We report the selection of $P_{\mathcal{B}}$ and $\lambda$ in Sec.~4.3 and Figures~7--8, and meta-training costs in Table~2. All datasets used in this work are publicly available. We will share an anonymized implementation with reviewers during the discussion period and release the code upon publication.

\ificlrfinal
\subsubsection*{Acknowledgments}
We thank Paul Janson, Abhinav Moudgil, and other members of the Belilovsky for helpful discussions and constructive feedback. This work was partially funded by FRQNT Doctoral (B2X) scholarship [B.T.] and FRQNT New Scholar grant [E.B.]. This work was also supported by computing resources from Mila (mila.quebec), Calcul Qu\'ebec, and the Digital Research Alliance of Canada.
The content is solely the responsibility of the authors and does not necessarily represent the views of the funding agencies.
\fi

\bibliography{ref}
\bibliographystyle{iclr2027_conference}

\clearpage
\appendix
\section{PES gradients remain unbiased when resuming from the buffer }
\label{apdx:pes}
Following [42], throughout this section $T$ denotes the unroll horizon and $N$ the number of Monte Carlo samples. In the following sections, For the reader's convenience, we will now restate background from~\citep{vicol2021pes} required to understand our proof. \citet{vicol2021pes} derive \emph{Persistent Evolution Strategies} (PES) for unrolled computation graphs whose total loss is
$L(\boldtheta)=\sum_{t=1}^T L_t(\boldtheta)$, where shared parameters $\boldtheta\in\mathbb{R}^P$ appear at every timestep.
To account for the contribution of each application of $\boldtheta$, they introduce per-timestep copies $\boldtheta_t=\boldtheta$
and stack them as $\Theta=(\boldtheta_1,\dots,\boldtheta_T)^\top$; writing $L_t(\bolds_t;\boldtheta)$ as
$L_t(\boldtheta_1,\dots,\boldtheta_t)$ (or simply $L_t(\Theta)$) implies $L_t(\cdot)$ is independent of $\boldtheta_\tau$ for $\tau>t$.
Let $\boldepsilon=(\boldepsilon_1,\dots,\boldepsilon_T)^\top$ be a matrix of perturbations with entries i.i.d.\ Gaussian of variance $\sigma^2$,
let $\otimes$ denote the Kronecker product, and define $\boldxi_t=\sum_{\tau=1}^t \boldepsilon_\tau$.
Starting from the full gradient $\pd{L(\Theta)}{\ovec{\Theta}}\in\mathbb{R}^{PT\times 1}$ and using ES,
\begin{align}
\frac{dL(\boldtheta)}{d\boldtheta} = \sum_{\tau=1}^T \pd{L(\Theta)}{\boldtheta_\tau} = (\mb I\otimes \boldone^\top)\pd{L(\Theta)}{\ovec{\Theta}}, \\
\boldg^{\text{PES}} = (\mb I\otimes \boldone^\top)\,\expect{\boldepsilon}{\frac{1}{\sigma^2}\,\ovec{\boldepsilon}\,L(\Theta+\boldepsilon)} \\= \frac{1}{\sigma^2}\expect{\boldepsilon}{\Big(\sum_{\tau=1}^T \boldepsilon_\tau\Big)L(\Theta+\boldepsilon)}.
\end{align}
This ES approximation is an unbiased estimator of the gradient of the Gaussian-smoothed objective
$\mathbb{E}_{\boldepsilon}[L(\Theta+\boldepsilon)]$, and it decomposes sequentially as
$\boldg^{\text{PES}}=\expect{\boldepsilon}{\sum_{t=1}^T \hat{\boldg}^{\text{PES}}_{t,\boldepsilon}}$ with
$\hat{\boldg}^{\text{PES}}_{t,\boldepsilon}=\frac{1}{\sigma^2}\boldxi_t\,L_t(\boldtheta_1+\boldepsilon_1,\dots,\boldtheta_t+\boldepsilon_t)$.
The Monte Carlo PES estimator is then
\[
\gpes=\frac{1}{N}\sum_{i=1}^N\sum_{t=1}^T \hat{\boldg}^{\text{PES}}_{t,\boldepsilon^{(i)}}.
\]
In practice, \citet{vicol2021pes} use antithetic sampling,
\[
\gpesanti
%= (\mb I\otimes \boldone^\top)\expect{\boldepsilon}{\frac{1}{2\sigma^2}\,\ovec{\boldepsilon}\Big(L(\Theta+\boldepsilon)-L(\Theta-\boldepsilon)\Big)}
\approx \frac{1}{2\sigma^2N}\sum_{i=1}^N\sum_{t=1}^T \boldxi_t^{(i)}\Big(L_t(\Theta+\boldepsilon^{(i)})-L_t(\Theta-\boldepsilon^{(i)})\Big).
\]
Finally, they show PES is exactly unbiased for quadratic objectives:

% if $\nabla_{\boldtheta}L(\boldtheta)$ exists and $L$ is quadratic so that
% \begin{align}
% L(\Theta+\boldepsilon)= L(\Theta)+\veps^\top\nabla_{\vTheta}L(\Theta)\\+\frac{1}{2}\veps^\top\nabla^2_{\vTheta}L(\Theta)\veps,
% \end{align}
% then $\text{\emph{bias}}(\gpesanti)=\mathbb{E}_{\boldepsilon}[\gpesanti]-\nabla_{\boldtheta}L(\boldtheta)=\boldzero$ \citep{vicol2021pes}.

% \begin{statement}[PES is unbiased restated from~\citep{vicol2021pes}] \label{statement:unbiased}
% The statement environment was already declared at the top of the file, so
% we reuse the existing numbering and reference it via Statement~\ref{statement:unbiased}.
% The following is a non-numbered restatement used inside the appendix.
\begin{statement}[PES is unbiased restated from~\citep{vicol2021pes}]\label{statement:unbiased}
    Let $\boldtheta \in \mathbb{R}^P$ and $L(\boldtheta) = \sum_{t=1}^T L_t(\boldtheta)$.
    Suppose that $\nabla_{\boldtheta} L(\boldtheta)$ exists, and assume that $L$ is quadratic, so that it is equivalent to its second-order Taylor series expansion:
    $
    L(\Theta + \boldepsilon) = L(\Theta) + \veps^\top \nabla_{\vTheta} L(\Theta) + \frac{1}{2} \veps^\top \nabla^2_{\vTheta} L(\Theta) \veps
    $.
    Then, $\text{\emph{bias}}(\gpesanti) = \mathbb{E}_{\boldepsilon}[\gpesanti] - \nabla_{\boldtheta} L(\boldtheta) = \boldzero$.
\end{statement}

When meta-training with ELO, we use an identical gradient estimator $\hat{\vg}^{\text{PES-A-Buffer}} \equiv \hat{\vg}^{\text{PES-A}}$, with the only difference being that some particles may have been resumed from the particle replay buffer. This has the effect of distancing the subsequent $\boldtheta_i$. However, by the definition of $\Theta$ the resulting gradient is unaffected by this distancing.

% \begin{corollary}[PES gradients from the replay buffer are unbiased]
% Mirroring the assumptions from Statement~\ref{statement:unbiased} and additionally assuming that the full PES state is preserved within the buffer, then $\text{\emph{bias}}(\hat{\vg}^{\text{PES-A-Buffer}}) = \mathbb{E}_{\boldepsilon}[\hat{\vg}^{\text{PES-A-Buffer}}] - \nabla_{\boldtheta} L(\boldtheta) = \boldzero$.
% \end{corollary}
% Use the existing corollary environment for appendix-aware numbering.

\begin{corollary}[PES gradients from the replay buffer are unbiased]\label{cor:buffer_unbiased}
Mirroring the assumptions from Statement~\ref{statement:unbiased} and additionally assuming that the full PES state is preserved within the buffer and subsequently resumed from, then $\text{\emph{bias}}(\hat{\vg}^{\text{PES-A-Buffer}}) = \mathbb{E}_{\boldepsilon}[\hat{\vg}^{\text{PES-A-Buffer}}] - \nabla_{\boldtheta} L(\boldtheta) = \boldzero$.
\end{corollary}

\textbf{Proof.}
Since Statement~\ref{statement:unbiased} has no temporal constraint on $\Theta$, it directly applies to our case where particles are restarted from the buffer. $\qed$

\section{Extended Background}\label{apdx:sec:relatedwork}
We now review additional backgrounds and related works beyond those discussed in section~\ref{sec:background}.

\paragraph{Input Features for Learned Optimizers.}
Taking the small\_fc\_lopt for example, which applies the same MLP independently to every
parameter, and it is shape-agnostic. For each parameter, it receives a
39-dimensional feature vector:

\begingroup
\renewcommand{\arraystretch}{1.6}
\setlength{\tabcolsep}{6pt}
\arrayrulecolor{black!18}
\resizebox{\textwidth}{!}{%
\begin{tabular}{@{} c c c @{}}
\textbf{Feature} & \textbf{Dim} & \textbf{Definition} \\
\midrule[1pt]
Gradient & 1 & $g_n$ \\
\midrule
Parameter value & 1 & $p_n$ \\
\midrule
Momentum & 3 &
$m_{\beta,n}=\beta m_{\beta,n-1}+(1-\beta)g_n,
\quad \beta\in\mathcal{B}=\{0.9,0.99,0.999\}$ \\
\midrule
Second moment & 1 &
$v_{\gamma,n}=\gamma v_{\gamma,n-1}+(1-\gamma)g_n^2,
\quad \gamma\in\mathcal{G}=\{0.95\}$ \\
\midrule
Momentum-preconditioned & 3 &
$m_{\beta,n}\bigl(v_{\gamma,n}+10^{-6}\bigr)^{-1/2}$ \\
\midrule
Inverse second moment & 1 &
$\bigl(v_{\gamma,n}+10^{-6}\bigr)^{-1/2}$ \\
\midrule
Normalized gradient & 3 &
$g_n\bigl(R_{\delta,n}/\operatorname{mean}(R_{\delta,n})\bigr)^{-1/2}
C_{\delta,n}^{-1/2},\quad \delta\in\mathcal{D}$ \\
\midrule
Row second moment & 3 &
$R_{\delta,n}=\delta R_{\delta,n-1}
+(1-\delta)\operatorname{mean}_{\mathrm{col}}(g_n^2),
\quad \delta\in\mathcal{D}$ \\
\midrule
Column second moment & 3 &
$C_{\delta,n}=\delta C_{\delta,n-1}
+(1-\delta)\operatorname{mean}_{\mathrm{row}}(g_n^2),
\quad \delta\in\mathcal{D}$ \\
\midrule
Inverse row second moment & 3 &
$\bigl(R_{\delta,n}+10^{-8}\bigr)^{-1/2},
\quad \delta\in\mathcal{D}$ \\
\midrule
Inverse column second moment & 3 &
$\bigl(C_{\delta,n}+10^{-8}\bigr)^{-1/2},
\quad \delta\in\mathcal{D}$ \\
\midrule
Factored-preconditioned momentum & 3 &
$m_{\beta,n}\bigl(R_{\delta,n}/\operatorname{mean}(R_{\delta,n})\bigr)^{-1/2}
C_{\delta,n}^{-1/2}$ \\
\midrule
Step embedding & 11 &
$\tanh(n/\tau-1),\quad
\tau\in\{1,3,10,30,100,300,10^3,3000,10^4,30000,10^5\}$ \\
\midrule
Total & 39 & --- \\
\end{tabular}%
}
\endgroup

The row/column statistics are computed at the tensor level and broadcast to
each element, so small\_fc still receives minor tensor-level information
despite operating element-wise. The network architecture is
$39\to32\to32\to2$ with ReLU activations. It outputs an update direction and
magnitude. The same interface is used for all parameter types, including
biases, embeddings, normalization parameters, convolution kernels, and
non-matrix tensors.

\paragraph{Learned Optimization.}
Learned optimizers (LOs) aim to replace hand-crafted optimizers with small neural networks. In incipient work,~\citet{andrychowicz2016learning} show that an LSTM can successfully be used to optimize small deep learning problems to a better loss than SGD or Adam. Follow-up research proposed new gradient estimators~\citep{metz2019understanding, vicol2021pes,vicol2023low,li2023variance}, introduced new architectures and features~\citep{metz2211velo,metz2022practical}, proposed low cost meta-training recipes~\citep{moudgil2025celo,moudgil2026celo2,therien2024mulo}, and introduced stabilizing mechanisms to improve generalization~\citep{strongerbaselines,harrison2022closer}. Of these works,~\citet{strongerbaselines} is the most closely related to our own as their proposed meta-training scheme also leverages behaviour cloning. However, unlike our work, they (1) behaviour clone using an alternating optimization scheme instead of adding the behavior cloning term to the meta-loss and (2) they do not use a replay buffer.

\paragraph{Replay Buffers.}
Replay buffers are widely used in reinforcement learning (RL)~\citep{ross2011reduction, liu2018effects,zhang2017deeper} and continual learning (CL)~\citep{rolnick2019experience, chaudhry2021using}. They store trajectories collected during training, enabling the learner to sample from past experiences rather than relying solely on the most recent data. In RL, this helps break temporal correlations~\citep{mnih2015human}, improves sample efficiency~\citep{schaul2015prioritized,schmitt2018kickstarting,vecerik2017leveraging,hester2018deep,nair2017overcoming}, and stabilize learning, while in CL it plays a key role in alleviating catastrophic forgetting~\citep{rolnick2019experience,buzzega2020dark}. Unlike conventional replay buffers for CL or RL, which typically contain existing training examples, we leverage a particle buffer that contains inner-problem checkpoints from previous unrolls, allowing existing particles to restart from the buffer instead of random initialization.

\paragraph{Behavior Cloning.}
Behavior Cloning (BC)~\citep{rajaraman2020toward,torabi2018behavioral} is one of the most widely used paradigms in imitation learning~\citep{pomerleau1988alvinn, osa2018algorithmic,liu2021curriculum}, where the objective is to approximate an expert policy by directly regressing from observed states to expert actions. Formally, given a dataset of expert demonstrations
$\mathcal{D} = \{(x_i, a_i^\star)\}_{i=1}^N$,
where $x_i \in \mathcal{X}$ denotes the state and $a_i^\star \in \mathcal{A}$ is the corresponding expert action, the goal is to learn a policy $\pi_\theta: \mathcal{X} \to \mathcal{A}$, parameterized by $\theta$, that closely approximates the expert's behavior. A simple version of behavior cloning solves the supervised regression problem~\citep{florence2022implicit}:
\begin{equation}
    \hat{\pi}_\theta = \arg\min_{\pi_\theta} \; \frac{1}{N}\sum_{i=1}^N \left\|\pi_\theta(x_i) - a_i^\star \right\|_2^2.
\end{equation}
In our work, we leverage progressive behavior cloning primarily as a stabilizer for buffer-based meta-training, enabling efficient long-unroll meta-training.

\par\noindent\begin{minipage}{\textwidth}
\section{Wall-clock Optimizer steps and meta-training time}

\captionsetup{type=table}
\centering
\caption{
\textbf{Wall-clock training time of different learned optimization algorithms on a single H100 GPU.}
Each algorithm is meta-trained for $100{,}000$ outer steps following the recipes outlined in Sec.~\ref{ss:experiments}. 
We report the median per-step time, computed as the median of consecutive \texttt{\_runtime} deltas logged on W\&B; total time is computed as median${}\times{}$steps. 
$\Delta$ is the relative overhead over the same-family baseline, either \texttt{small\_fc} or \texttt{Celo2}. 
Although these runs are measured on H100, our meta-training tasks are small and the runtime is mainly \textbf{limited by CPU operations and data movement}, so using more accessible GPUs such as L40S gives nearly unchanged wall-clock time. 
As the per-step optimizee computation increases, e.g., in the Celo2 setting, ELO's relative overhead becomes substantially smaller.
}
\label{tab:meta_train_runtime}
\begin{tabular}{@{}llrrr@{}}
\toprule
Method                          & LO arch.    & Median (s/step) & Total wall-time & $\Delta$ \\
\midrule
\textsc{small\_fc}                       & small\_fc  & 0.116           & 3h14m           & --      \\
\textsc{CL-small\_fc}                    & small\_fc  & 0.122           & 3h23m           & +4.6\%    \\
\textsc{ELO-small\_fc}                   & small\_fc  & 0.158           & 4h23m           & +35.6\%   \\
\midrule
\textsc{Celo2}                           & Celo2       & 0.187           & 5h11m           & --      \\
\textsc{ELO-Celo2}                       & Celo2       & 0.231           & 6h25m           & +23.8\%  \\
\bottomrule
\end{tabular}
\end{minipage}\par

% \section{Extended background}
% \textbf{Meta-training Objective} In general, learning the meta-parameters, $\phi$, involves solving an optimization problem of the form:
% \begin{equation}\label{eq:obj}
% \min_{\phi} \; \mathbb{E}_{\tau \sim \mathcal{D}} \Big[ \sum_{n=1}^N \mathcal{L}^{\text{meta}}_n(\theta_n; \phi, \tau) \Big],
% \end{equation}
% where $\mathcal{D}$ is a distribution of tasks, $\phi$ are the meta-parameters, $\tau$ specifies a task (e.g., optimizee initialization, dataset, objective), and $\theta_n$ evolves as a function of $\phi$ as described in eq~\ref{eq:lo_update}. The objective seeks to minimize the sum of per-timestep losses over the training horizon $N$~\citep{therien2024mulo}.

%We will show that the same proof holds when using a particle replay buffer to resume meta-training at a later stage.

\section{Extended Empirical Settings}\label{apdx:ss:experiments}
This section complements sec.~\ref{ss:experiments} from the main manuscript with additional experimental details.

\paragraph{Meta-training.} 
% Our experimental setup and meta-training pipeline largely follow~\citep{metz2022practical, therien2024mulo,moudgil2025celo} with best hyperparameters applied for each baseline. All learned optimizers are meta-trained on four simple image classification tasks, MNIST, Fashion-MNIST, CIFAR-10, and SVHN, resized to $8 \times 8$. The optimizee is a one-layer MLP with width 32. We meta-train for 100K outer steps with batch size 64, using AdamW as the meta-optimizer with a cosine learning rate schedule. We search the outer learning rate over $\{10^{-3}, 3 \times 10^{-4}, 10^{-4}\}$ and use a weight decay of $10^{-4}$. We estimate meta-gradients with Persistent Evolution Strategies (PES), using a truncation length of $K=50$. We also apply task augmentation with range of [0.001, 1000] during meta-training, following~\citep{moudgil2025celo}.

For \texttt{small\_fc\_lopt} and Celo2, we sample the unroll length from $[100, 2000]$ using log-uniform sampling. For CL-\texttt{small\_fc\_lopt}, we use Adam as the expert optimizer. We evaluate maximum meta-training unroll lengths of 10K and 30K, and find that 10K is the most stable setting and gives the best downstream performance. For ELO variants, we sample the base unroll length uniformly from $[100, 2000]$ and use the proposed failure-aware resume buffer to resume from difficult regions of the current optimization trajectory. We first evaluate learned optimizer checkpoints on ImageNet-1K ($32 \times 32$), using a three-layer MLP with width 128 as the optimizee. Each optimizer is evaluated for 10K inner steps with batch size 4096. This evaluation serves as an intermediate meta-evaluation benchmark before transferring the learned optimizers to larger downstream workloads.

\paragraph{Image classification.}
For realistic vision pre-training, we evaluate on ImageNet-1K at $224 \times 224$ resolution using ResNet-50~\citep{he2016deep} and ViT-Base/16~\citep{dosovitskiy2020image}. We compare learned optimizers mainly against AdamW~\citep{loshchilov2017decoupled}, which remains the dominant trend for vision tasks. Training runs for 50K steps with batch size 2048 under a cosine learning-rate schedule and weight decay searched from \{0.1, 0.01, 0.001, 0.0001\}. For AdamW, we sweep seven learning rates on a logarithmic grid in $[10^{-2}, 10^{-5}]$ for ViT-Base/16 and in $[10^{-1}, 10^{-4}]$ for ResNet-50. For all learned optimizers, we sweep seven learning rates logarithmically in $[10^{-2}, 10^{-5}]$. We also sweep the data augmentation strength for each method, ranging from light augmentation (e.g., horizontal flip, random crop) to stronger combinations (e.g., randomly select six augmentations from a pool of more than twenty operations).

\paragraph{Language model pre-training.}
For language modeling, we compare learned optimizers with AdamW and the matrix-based optimizer Muon~\citep{jordan2024muon}. We pretrain GPT-2 models~\citep{radford2019language} with 30M and 124M parameters, as well as GPT-3 (1.3B), on FineWeb~\citep{penedo2024fineweb}. The GPT-2 models use token budgets of 0.6B and 2.5B, respectively, corresponding to approximately 20 tokens per parameter~\citep{chinchilla}. The 30M model is trained for 18,311 steps with an effective batch size of 64 and a sequence length of 512, while the 124M model is trained for 19,074 steps with an effective batch size of 128 and a sequence length of 1024. For GPT-3 (1.3B), we use sequence length 2048 and an effective batch size of 256, training for 19,274 steps on approximately 10B tokens. All methods use a cosine learning rate schedule. For each method, we tune weight decay on GPT-2 (124M) pretraining, selecting from $\{0.1, 0.01, 0.001, 0.0001\}$. We sweep seven logarithmically spaced learning rates in $[10^{-5}, 10^{-2}]$ for learned optimizers and AdamW, and in $[10^{-4}, 10^{-1}]$ for Muon.

%To aid the reader's understanding, we illustrate the meta-training process of a learned optimizer in \textbf{Figure}~\ref{fig:meta-training-pipline}.

% \subsection{Sweep of Direction and Magnitude Imitation Weights}

% \label{app:dir_mag_sweep}

% We study the effect of the direction and magnitude weights in the imitation objective. 

% The two terms provide different supervision signals: the direction loss encourages the learned optimizer to follow the expert update direction, while the magnitude loss controls the update scale. 

% As shown in \textbf{Figure}~\ref{fig:dir_mag_sweep}, placing more weight on direction generally gives lower meta-validation loss, while using too much magnitude supervision weakens meta-training. 

% This supports our choice of decoupling direction and magnitude instead of directly regressing the full expert update with a single loss.

\begin{figure}[htbp]
    \centering
    {\includegraphics[width=1.0\linewidth]{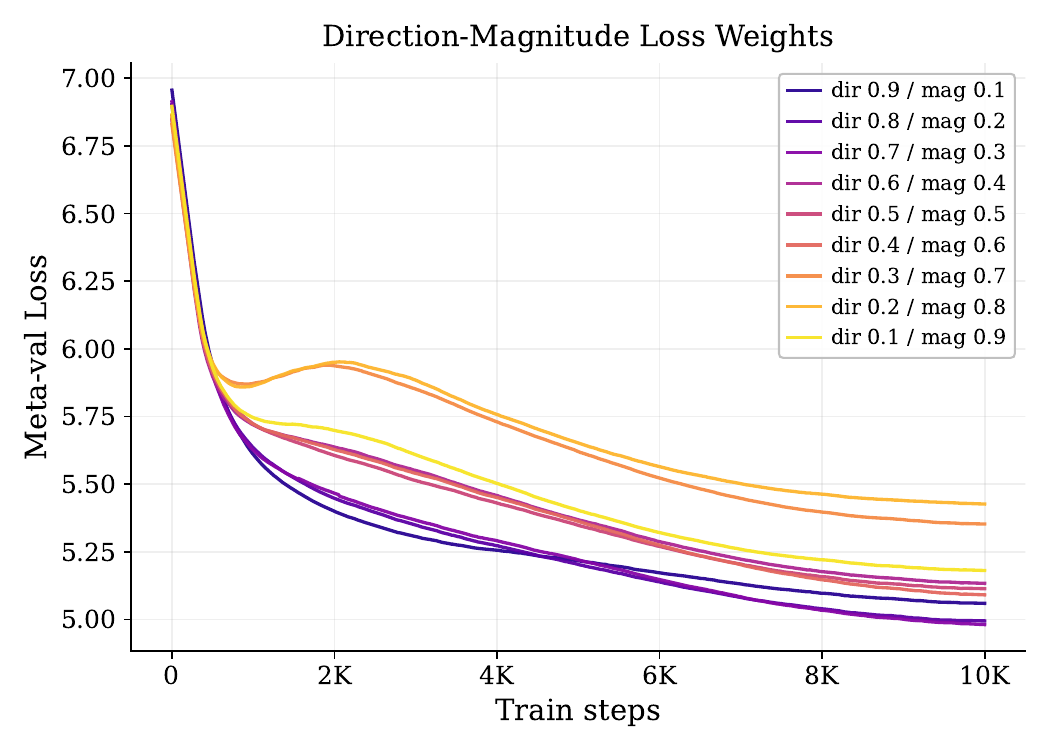}}
\caption{
Effect of direction and magnitude weights in the imitation objective.
We sweep the weighting between direction loss and magnitude loss and report meta-validation loss during training.
Direction-dominant objectives consistently achieve lower meta-validation loss, while overly large magnitude weights slow convergence and lead to worse final performance.
This supports decoupling update direction and scale rather than using a single regression loss on the full expert update.
}
\label{fig:dir_mag_sweep}
\end{figure}

% \begin{figure}[htbp]
%     \centering
%     \includegraphics[width=1.0\linewidth]{figs/dir_mag_sweep.pdf}
%     \caption{
%     Sweep of direction and magnitude weights in the imitation objective. 
%     Direction-dominant weighting generally yields lower meta-validation loss, indicating that update direction provides a stronger supervision signal than update magnitude in early meta-training.
%     }
%     \label{fig:dir_mag_sweep}
% \end{figure}

\begin{figure}[htbp]
    \centering
    {\includegraphics[width=1.0\linewidth]{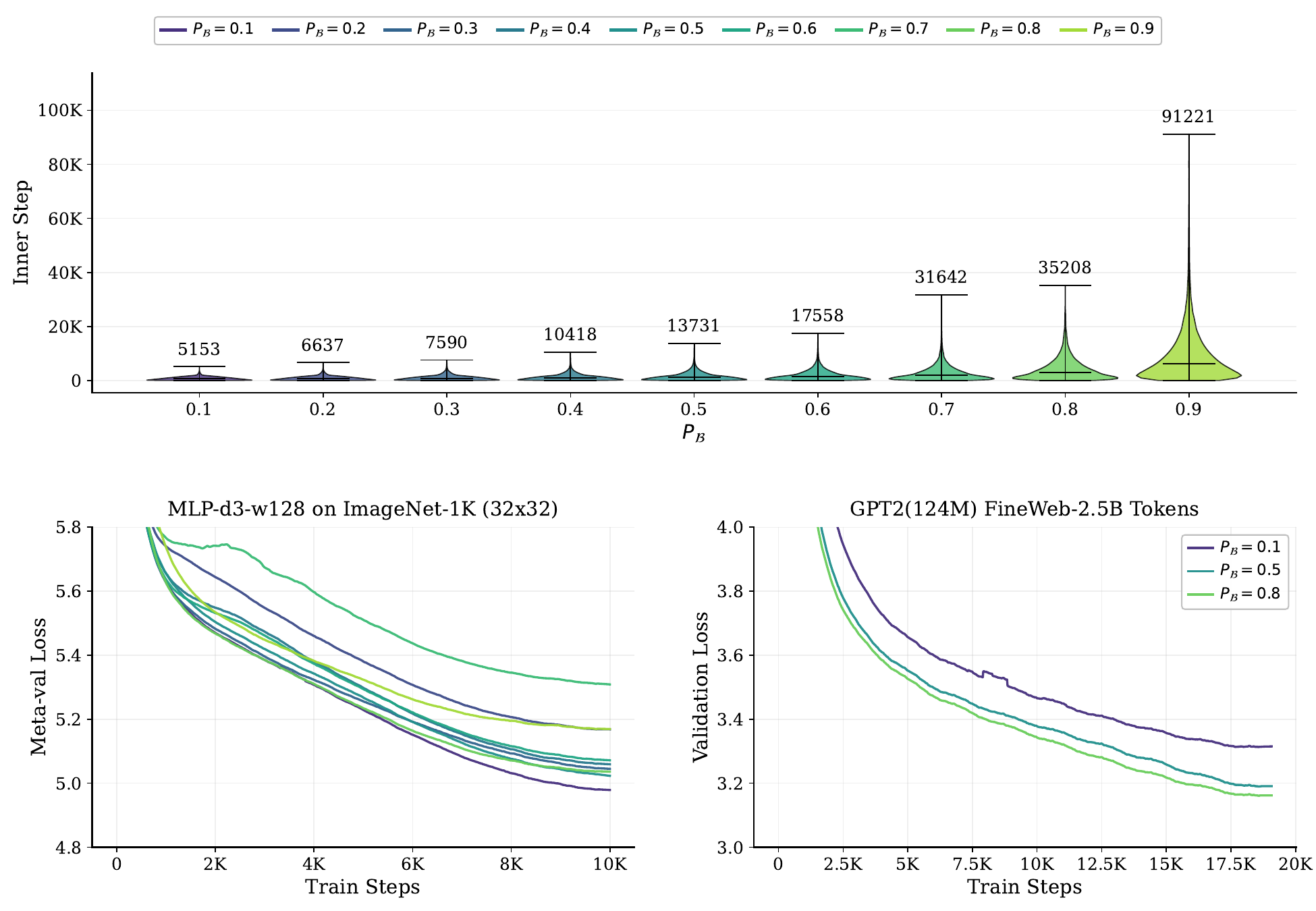}}
\caption{
Effect of the resume probability $P_{\mathcal B}$ on long-horizon sampling and downstream performance.
Top: increasing $P_{\mathcal B}$ shifts the sampled inner steps toward later regions of the trajectory, with the maximum reached step increasing from 5,153 at $P_{\mathcal B}=0.1$ to 91,221 at $P_{\mathcal B}=0.9$.
Bottom left: on the ImageNet-1K $32 \times 32$ meta-evaluation task, smaller $P_{\mathcal B}$ generally gives better meta-validation loss, with the lowest loss achieved at $P_{\mathcal{B}} = 0.1$
Bottom right: on GPT-2 pre-training with 124M parameters on FineWeb for 2.5B tokens, larger $P_{\mathcal B}$ improves validation loss, indicating that practical long-horizon workloads benefit from more frequent resume-based sampling.
}
\label{fig:buffer_prob_sweep}
\end{figure}

\clearpage
\end{document}